\documentclass[letterpaper, 10 pt, conference]{ieee/ieeeconf}  

\IEEEoverridecommandlockouts                              

\usepackage{amsmath} 
\usepackage{amssymb}  
\usepackage{soul}
\usepackage{tikz}
\usepackage{mathtools}
\usepackage{mpac_tikz_tools}
\usepackage{array}
\usepackage{caption}
\usepackage{subcaption}
\usepackage{algorithm}  
\usepackage[noend]{algpseudocode}  
\usepackage{centernot}
\usepackage[mathscr]{euscript}

\usepackage{amsthm}
\theoremstyle{definition}

\newtheorem{proposition}{\normalfont\bfseries Proposition}
\newtheorem{definition}{\normalfont\bfseries Definition}

\newtheorem{remark}{\normalfont\bfseries Remark}

\newtheorem{manualtheoreminner}{}
\newenvironment{primitive}[1]{%
  \renewcommand\themanualtheoreminner{#1}%
  \manualtheoreminner
}{\endmanualtheoreminner}

\newcommand{\motionprimitive}{\mathcal{P}}

\newcommand{\textbfit}[1]{\textbf{\textit{{#1}}}}

\usepackage{environ}

\NewEnviron{ruledtable}{%
\par\addvspace{2mm}\hrule height 0.03cm 
\begin{table}[!h]\BODY\end{table}
\hrule height 0.03cm \addvspace{2mm}
}

\title{\LARGE \bf
 Robust Locomotion on Legged Robots through Planning on
 \\ Motion Primitive Graphs
}

\author{Wyatt Ubellacker  and Aaron D. Ames%
        \thanks{This research is supported by Dow (\#227027AT). \vspace{1mm}}%
\thanks{Authors are with the Departments of Control and Dynamical Systems, Mechanical and Civil Engineering, California Institute of Technology, Pasadena, CA, USA.
{\tt\small wubellac, ames@caltech.edu }}
}

\begin{document}
\nocite{video} 

\maketitle
\thispagestyle{empty}
\pagestyle{empty}

\begin{abstract}
The functional demands of robotic systems often require completing various
tasks or behaviors under the effect of disturbances or uncertain environments. Of increasing
interest is the autonomy for dynamic robots, such as multirotors, motor
vehicles, and legged platforms. Here, disturbances and environmental conditions
can have significant impact on the successful performance of the individual dynamic
behaviors, referred to as ``motion primitives''.
Despite this, robustness can be achieved by switching to and
transitioning through suitable motion primitives. This paper contributes such a
method by presenting an abstraction of the motion primitive dynamics and a
corresponding ``motion primitive transfer function''. From this, 
a mixed discrete and continuous ``motion primitive graph'' is constructed, and an algorithm
capable of online search of this graph is detailed.
The result is a framework capable of realizing holistic
robustness on dynamic systems. This is experimentally demonstrated for a
set of motion primitives on a quadrupedal robot, subject to various
environmental and intentional disturbances.

\end{abstract}
\vspace{-1mm}

\section{Introduction}
\label{sec:introduction}
There is a wealth of research and applications of functional autonomy and
demonstrations on robotic systems that range from highly structured
manufacturing applications \cite{PEDERSEN2016282} to exploring the alien
environments on other planets \cite{autonomy_mars_2020, autonomy_europa_2021}.
This autonomy is often realized by sequences
\cite{TeamRobosimianKaruma2017,edelbergCASAH}, state-machines
\cite{Backes_IRSA_2018}, or graph-search \cite{kim2001executing} autonomy to chain
behaviors together to perform complex objectives. In many applications,
including autonomous vehicles, human-robot interactions, and dynamic legged
robots, robustness to uncertainties and disturbances is critical to successful
function. 

\begin{figure}
  \centering
  \includegraphics[width=\columnwidth]{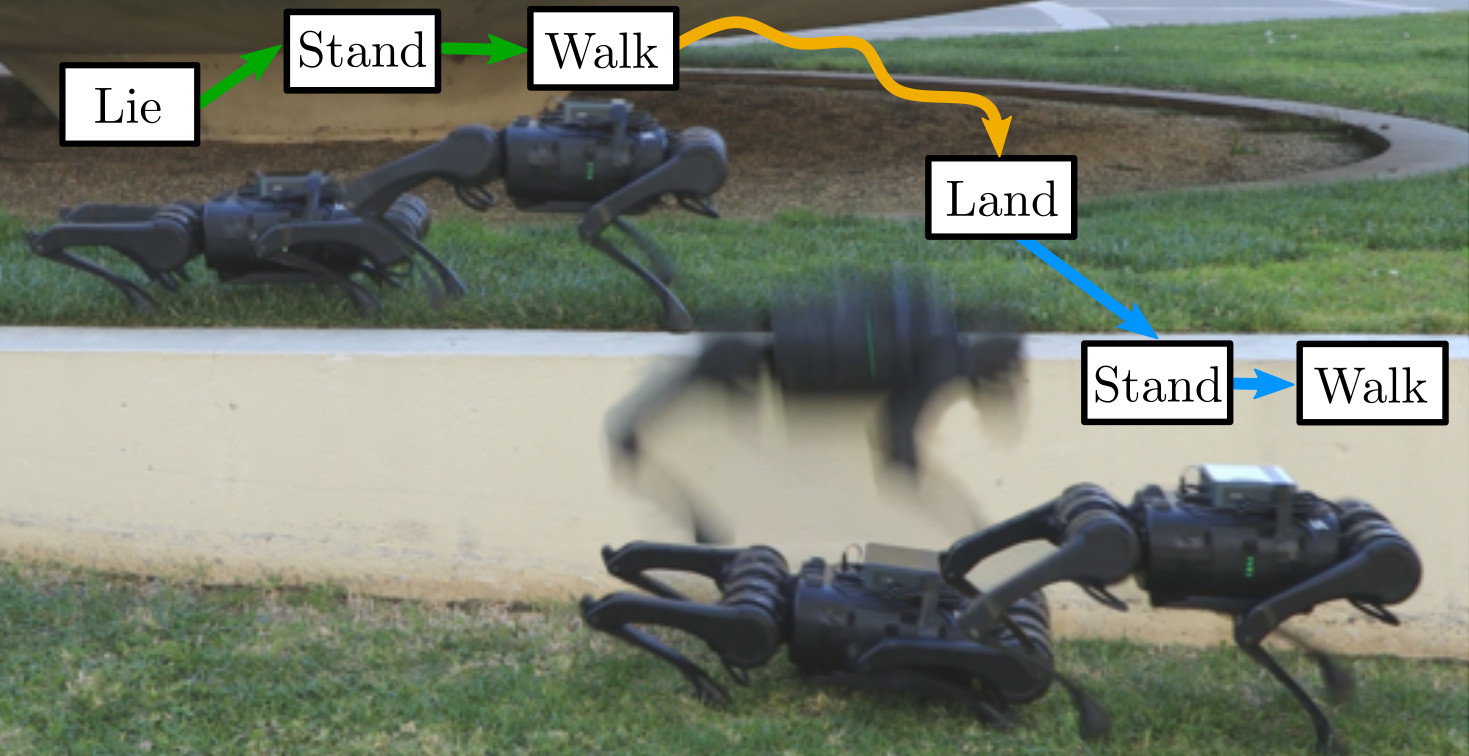} \\
  \vspace{2mm}
  \includegraphics[width=\columnwidth]{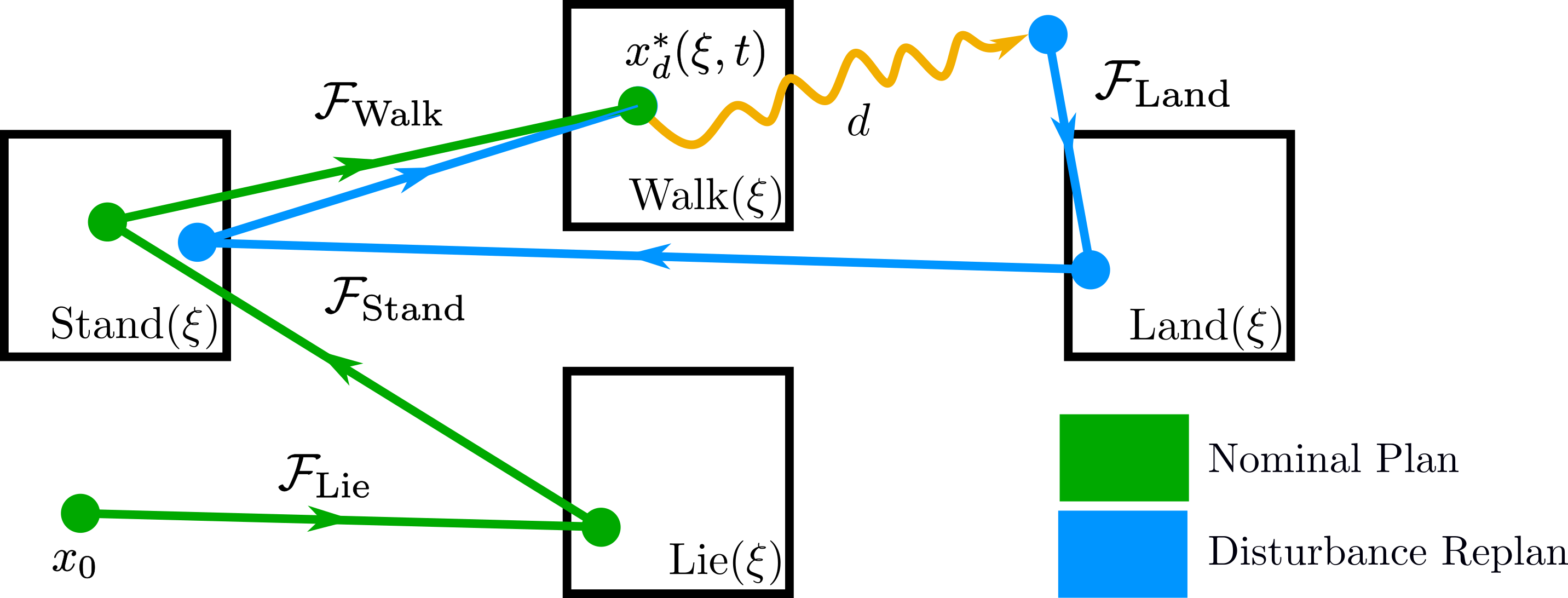}
  \caption{A quadrupedal robot demonstrating robustness to falling off a ledge by
continuously recomputing motion primitive transitions as disturbances interrupt nominal operation.}
  \vspace{-3mm}
  \label{fig:intro_fig}
\end{figure}
There are extensive studies of robust autonomy on dynamic systems
\cite{bdi_atlas,Park_fsm_walking,sales2014adaptive,singletary2020safety}.
However, success typically relies on transition-specific analysis or heuristic
conditions to determine switching behavior between dynamic primitive behaviors (commonly referred to as
\textit{motion primitives} \cite{HumanInspiredZhao2014,motion_prim_Paranjape,motion_prim_Kolathaya,
Ubellacker2021}). There is significant active work developing
individual motion primitives for various dynamic systems
\cite{kim2019highly,gomez2001parallel,falanga2017vision},
and capability is rapidly increasing in both volume and complexity.
If we are to effectively incorporate increasingly complex behaviors into a
dynamic autonomous system, we require a more formal method to determine
appropriate transitions, both in nominal operation and in response to
disturbances.

With this motivation, we build upon previous work on motion primitive
transitions \cite{Ubellacker2021} where only discrete motion
primitives and offline search were considered, and is limited in usefulness in
regards to robustness. This manuscript extends the definition of motion
primitives to include continuous arguments, develops a method for online
search, and provides a methodology to manage the resulting complexity. We rely
on notions
of stability and regions of attraction to determine transition conditions
\cite{tedrake2009lqr,sousa2020command} and construct an
abstraction of the dynamics that captures the mapping of
dynamic state across the application of a motion primitive. This leads
to a natural mixed discrete and continuous \textit{motion primitive graph} that
scales in complexity with number of primitives and associated arguments rather
than the system dynamics.  Inspired by the success of probabilistic search on
similar problems,
\cite{Lavalle98rapidly-exploringrandom,probabilistic_roadmap}, we propose a
motion primitive graph search algorithm capable of continuous planning towards
a desired motion primitive in both nominal and disturbed conditions. This
summary represents the main contribution of this paper -- a method to plan
through motion primitive transitions despite underlying dynamics and
complexity.

This procedure is applied to a quadrupedal
robot with a set of motion primitives. Several experiments across a variety of
environmental and antagonistic disturbances are successfully performed, and the results
and accompanying video highlight the contributions of this work.

\section{Preliminaries}
\label{sec:preliminaries}
For the duration of this manuscript, we consider a nonlinear system in control
affine form. We have system dynamics
\begin{align}
\label{eq:dynamics}
\dot{x} = f(x) +g(x)u
\end{align}
with state $x\in \mathcal{X} \subset \mathbb{R}^n$ and control inputs $u\in
\mathcal{U}\subset\mathbb{R}^n$. The functions ${f:\mathcal{X}\to \mathbb{R}^n}$
and ${g:\mathcal{X}\to \mathbb{R}^{n\times m}}$ are assumed to be locally
Lipschitz continuous. 
The \textit{flow} of this system, $\phi_t(x_0)$, is 
solution to the initial value problem with $x(0)=x_0$. The Lipschitz
assumptions give that $\phi_t(x_0)$ is unique and, assuming forward completeness,
exists for all $t \geq 0$.

\subsection{Motion Primitives}
While the idea of motion primitives is not new, we 
introduce our own definition specifically suited for our purposes.
This definition is generalization of the definition from
previous work \cite{Ubellacker2021} to a larger class of motion primitives. 

\begin{definition}{}
\label{def:primitive}
A \textbfit{motion primitive} is a dynamic behavior of~(\ref{eq:dynamics})
defined by the 6-tuple
$\motionprimitive={(\Xi,x^*,k,\Omega,\mathcal{C},\mathcal{S})}$ with the
following attributes:
\begin{itemize}
\item The valid \textit{arguments}, $\Xi \subset \mathbb{R}^a$. This is
the bounded set of continuous arguments that specifies the motion primitive's
behavior.
\item The \textit{setpoint}, ${x^*(x_0,\xi,t):\mathcal{X} \times \Xi \times
\mathbb{R} \to \mathcal{X}}$, that describes the desired state as a function of
initial state, arguments, and time.
It satisfies~(\ref{eq:dynamics}) and hence ${x^*(x_0,\xi,t+t_0)=\phi_t(x^*(x_0,\xi,t_0),\xi)}$, ${\forall t \geq 0, t_0 \in \mathbb{R}}$.
It may be executed with a selected initial time $t_0$. 
$x^*$ must be differentiable with respect to $\xi$ and $t$ over the safe region of attraction
($\mathcal{S}(x^*(\cdot),\xi)$, as defined below).
\item The \textit{control law}, ${k:\mathcal{X} \times \Xi  \times \mathbb{R}
\to \mathcal{U}}$, that determines the control input ${u=k(x,\xi,t)}$.
It is assumed to render the setpoint locally exponentially stable on the region
of attraction ($\Omega(x^*(\cdot),\xi)$, as below).
For constants $M, \alpha > 0 \in \mathbb{R},$ all $t>t_0$ and $x_0 \in 
\Omega(x^*(\cdot),\xi_0)$ implies:
\begin{align} \label{eq:exp_stab} 
||\phi_{t-t_0}(x_0) - x^*(\cdot)|| \leq Me^{-\alpha(t-t_0)}||x_0 - x^*(\cdot)||
\end{align}
\item The \textit{region of attraction (RoA)} of the setpoint, ${\Omega:
\mathcal{X} \times \Xi \to \mathscr{P}(\mathcal{X})}$, given by ${\Omega(x^*(\cdot),\xi) \subseteq \mathcal{X}}$:
\begin{align*}
    \Omega(x^*(\cdot),\xi) = \{&x_0\in\mathcal{X} :\\&\lim_{t\to \infty} \phi_t(x_0) \!-\!
x^*(x_0,\xi,t\!+\!t_0) = 0\}.
\end{align*}
\item The \textit{safe set}, ${\mathcal{C}: \mathcal{X} \times \Xi \to \mathscr{P}(\mathcal{X})}$, that
indicates the states of safe operation by the set ${\mathcal{C}(x^*(\cdot),\xi) \subset \mathcal{X}}$. It is assumed that
${x^*(x_0,\xi,t+t_0) \in \mathcal{C}(t)}$, ${\forall t \geq 0, t_0 \in \mathbb{R}}$.
This is designer-specified, and it is important to note that
${x \in \mathcal{C}(x^*(\cdot),\xi) \centernot\Longleftrightarrow x \in \Omega(x^*(\cdot),\xi)}$.
\item The \textit{safe region of attraction}, ${\mathcal{S}:
\mathcal{X} \times \Xi \to \mathscr{P}(\mathcal{X})}$, that defines
the set of states from which the flow converges to the setpoint while being safe for all time:
\begin{align*}
\mathcal{S}(x^*(\cdot),\xi)=\{&x_0 \in \Omega(x^*(\cdot),\xi):\\ &\phi_t(x_0,\xi,t) \in \mathcal{C}(x^*(\cdot),\xi), \forall t \geq 0\}.
\label{eq:safeRoA}
\end{align*}
\end{itemize}
\end{definition}

A illustration of the relationship between motion primitive attributes can be seen in Figure~\ref{fig:primitive_relationship} and elucidating examples can be found in 
Section~\ref{subsec:quadruped_primitives}.

\begin{figure}
  \centering
  \includegraphics[width=0.9\columnwidth]{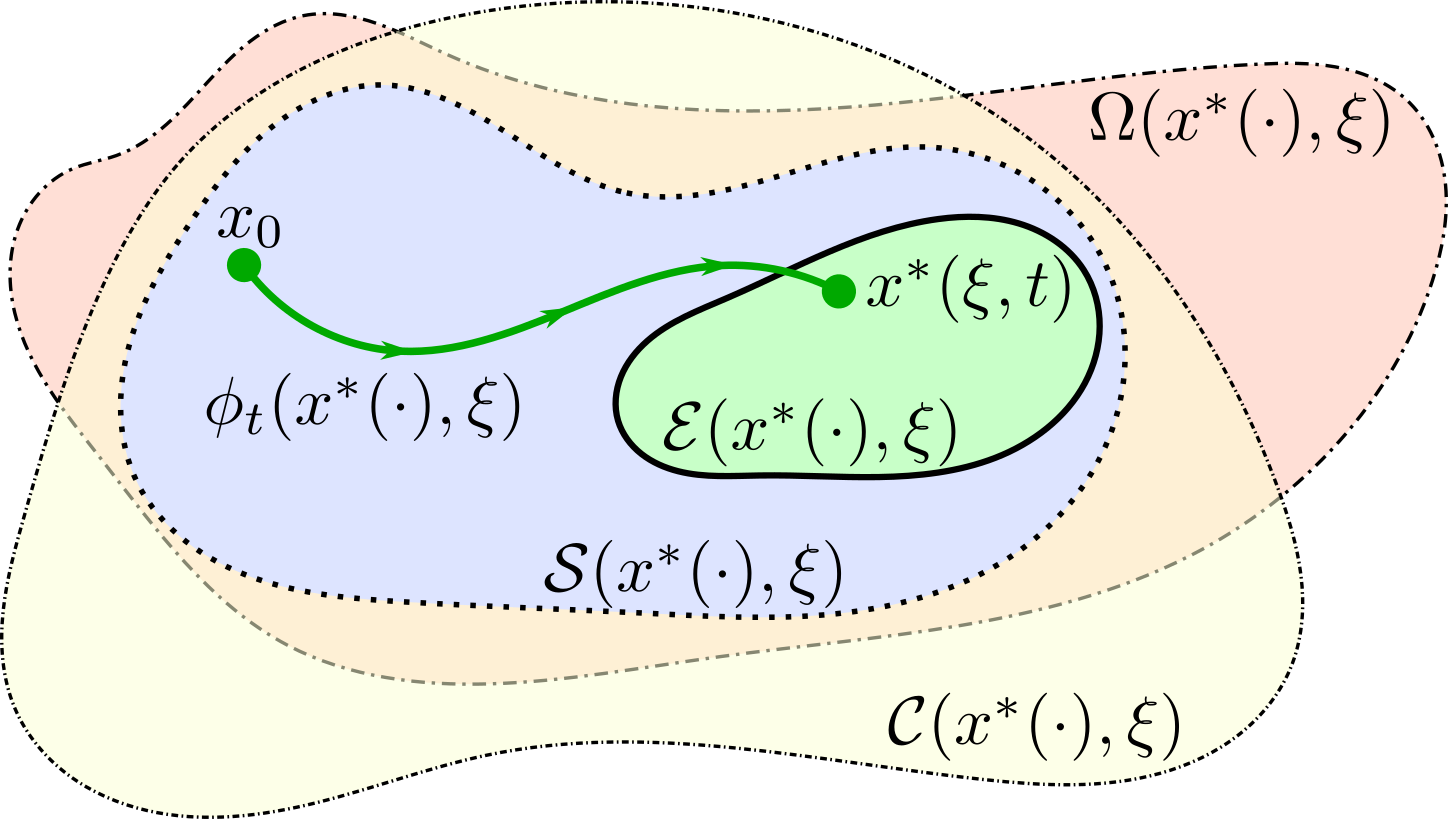}
  \caption{Depiction of motion primitive attributes and their relationships. Note the dependence on $\xi$ and $x^*(x_0,\xi,t)$.}
  \label{fig:primitive_relationship}
\end{figure}
\subsection{Motion Primitive Transfer Function}
\begin{proposition}
Consider a motion primitive with argument $\xi_0 \in \Xi$, 
initial condition $x_0 \in \mathcal{X}$ and setpoint
$x^*(x_0,\xi_0,t)$. If $x_0 \in \mathcal{S}(x^*,\xi_0)$. There exists a duration
$\Delta t_{min} \geq 0, \Delta t_{min}\in \mathbb{R}$ for any small constant $\epsilon > 0, \epsilon \in \mathbb{R}$ such that:
\begin{equation} \label{eq:epsilon_prop}
||\phi_{\Delta t_{min}}(x_0) - x^*(x_0,\xi_0,t_0+\Delta t_{min})|| < \epsilon
\end{equation}
\end{proposition}

\begin{proof}
Consider the control law for the primitive $k(x,\xi_0,t)$. 
This control law is assumed to render $x^*(x_0,\xi_0,t)$
exponentially stable over $\Omega(x^*(\cdot),\xi_0)$.
Take constants $M,\alpha > 0 \in \mathbb{R}$ satisfying Inequality~\ref{eq:exp_stab} and
$\epsilon >0, \epsilon \in \mathbb{R}$. Consider:
\begin{align*}
Me^{-\alpha(t-t_0)}&||x_0 - x^*(\cdot)|| < \epsilon \\
\Rightarrow t-t_0 &> \underbrace{-\frac{1}{\alpha}\log\left(\frac{\epsilon}{M||x_0 -
x^*(\cdot)||}\right)}_{\Delta t_{min}}
\end{align*}
As $x_0 \in \mathcal{S}(x^*,\xi_0) \subseteq \Omega(x^*,\xi_0)$,
Inequality~\ref{eq:exp_stab}~$\Rightarrow$~Inequality~\ref{eq:epsilon_prop} and we have the
existence of $\Delta t_{min}$ as desired.
\end{proof}
 Choosing $\epsilon$ so that the deviation is negligible (in practice, within
the accuracy of sensing), allows us to build an abstraction
of the motion primitive dynamics that we call the \textit{motion primitive transfer
function}.

\begin{figure}
  \centering
  \includegraphics[width=0.8\columnwidth]{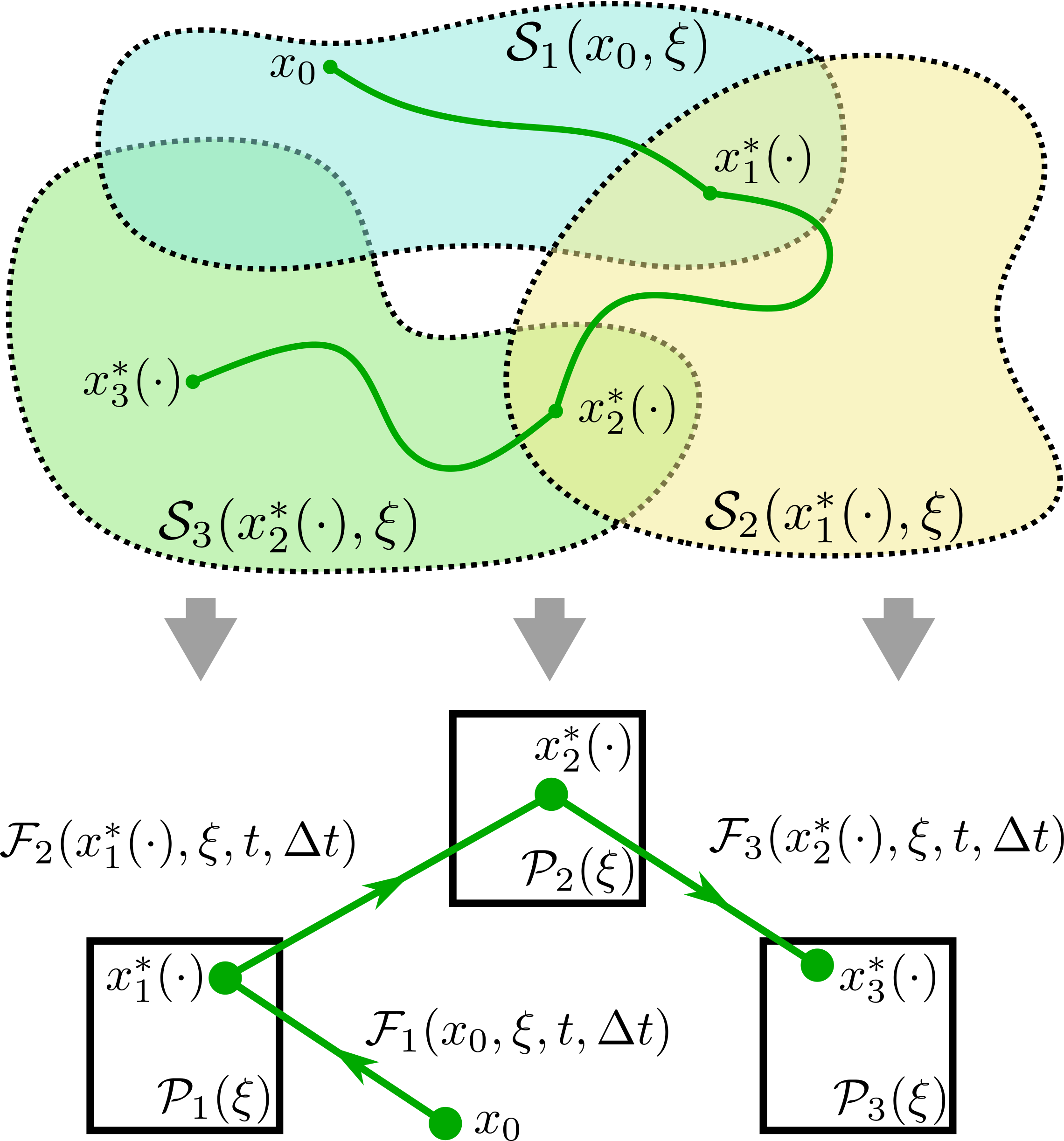}
  \caption{Dynamics of and relationships between motion primitives are
abstracted via the \textit{motion primitive transfer function}, implicitly producing a
mixed discrete and continuous graph.
}
  \label{fig:map_to_graph}
\end{figure}

\begin{definition}{}
\label{def:motion_primitive_transfer}
A \textbfit{motion primitive transfer function} is a map $\mathcal{F}:
\mathcal{X} \times \Xi \times \mathbb{R} \times \mathbb{R} \to \mathcal{X}$ that
abstracts the dynamics of a motion primitive from input system state, arguments,
initial time, and duration to an output system state:
\begin{align}
\mathcal{F}(x_0,\xi,t_0,\Delta t) = 
\begin{cases} 
      x^*(x_0,\xi,t_0+\Delta t) & \text{if } x_0 \in
\mathcal{S}(x^*), \\ &\Delta t \geq \Delta t_{min}\\
      x_0 & \text{otherwise}
\end{cases}
\raisetag{0.4cm}
\end{align}
\end{definition}

If the motion primitive is safe to use and our abstraction is valid, then this
function returns the setpoint of the primitive. If it is unsafe or the duration
is too short to ignore transient behavior, then the motion primitive cannot be
applied, and the map simply returns the state unchanged.

\begin{remark}
The motion primitive transfer function is composable in $x$, and a transition to
a specific motion primitive from a arbitrary state can be found by chaining
together motion primitive transfer functions, e.g. a sequence of transfer
functions such that:
\begin{align*}
x^*_1(x_0,\xi_1,t_1) &= \mathcal{F}_1(x_0, \xi_1,t_1,\Delta t_1)\\
x^*_2(x^*_1,\xi_2,t_2) &= \mathcal{F}_2(x^*_1, \xi_2,t_2,\Delta t_2)\\[-5pt]
&\;\;\vdots \\[-5pt]
x^*_n(x^*_{n-1},\xi_n,t_n) &= \mathcal{F}_n(x^*_{n-1}, \xi_n,t_n,\Delta t_n)
\end{align*}
\end{remark}

This construction builds a natural \textit{motion primitive graph} structure in
the state space of our system and the transition from an arbitrary state to a
specific motion primitive is reduced to finding an appropriate path of motion
primitive transfers. Namely, $\mathcal{R} = \{\mathcal{F}_i(\cdot,\xi_i, t_i, \Delta t_i)$,
$i=0,\ldots,n\}$ where $n$ corresponds to the desired motion primitive.

\section{Searching Mixed Discrete and Continuous Graph}
\label{sec:contribution}
In order to achieve robustness via motion primitive transitions, valid
motion primitive transition paths used to react to disturbances and
environmental uncertainties must be found in real time. Despite our abstraction
of the motion primitive dynamics, searching quickly still poses a challenge.
Typical methods, such as discrete search and pruning-based methods, scale
unfavorably with the number motion primitives, size of the argument sets, and
the dimensionality of the state space. Additionally, there is no expectation of
convexity in the motion primitive transfer functions, posing difficulty for
optimization-based methods.

There has been demonstrable success solving this class of search by using
randomized search algorithms, including Rapidly-exploring Random Trees (RRT)
\cite{Lavalle98rapidly-exploringrandom}, and variants.
This class of algorithm can effectively search
high-dimensional, nonconvex space, and is a natural choice for searching our
motion primitive graph. We leverage this with an RRT-based search to discover
feasible paths in the motion primitive graph, followed by constrained
gradient descent and node-pruning post-processing.

\subsection{RRT-based Feasible Path Search Algorithm}
\begin{algorithm}[t]
\caption{FeasiblePathSearch}
\begin{small}
\begin{algorithmic}[1]
\State \textsc{\textbf{Node :=}} \{state, action, parent, cost to come, est. cost to go\}

\Function{FeasiblePathSearch ($\motionprimitive_d(\xi_d), x_0$)}{}
    \State $\mathcal{R} = \{\}$
    \State $n_d = \textsc{Node}(\emptyset)$
    \State $J_{\text{g}} = \textsc{Cost}(x^*_d(\cdot), x_0$)
    \State $n_0 = \textsc{Node}(x_0,\emptyset,\emptyset,0,J_\text{g})$
    \State $\mathcal{N} = \{n_0\}$
    \State $(n_c, c) = (n_0, J_\text{g})$
  \While{\textsc{Node}$(x^*_d(\cdot), \mathcal{F}_d(\cdot, \xi_d, \cdot), \cdot) \not\in \mathcal{N}$}
    \State $n_s = \textsc{Node}(\emptyset)$
    \State $u \thicksim U([0,1])$
    \If {$u < p$} \Comment{Sample the cheapest node}
      \State $n_s = n_c$
    \Else \Comment{Draw uniformly from existing nodes}
      \State $n_s \thicksim U(\mathcal{N})$
    \EndIf
    \State $\mathcal{F}_s \thicksim U(\mathcal{F})$ \Comment{Draw uniformly from transfer functions}
    \State $\xi_s \thicksim U(\Xi_s)$
    \State $t_s \thicksim U([t_{s_{min}}, t_{s_{max}}])$
    \State $\Delta t_s \thicksim U([\Delta t_{s_{min}}, \Delta t_{s_{min}} + t_{s_{max}}])$
    \State $x^*_s = \mathcal{F}_s(n_s.x, \xi_s,t_s,\Delta t_s)$
    \If{$n_s.x \neq x^*_s$} \Comment{Add new node if $\mathcal{F}_s$ progresses}
      \State $J_{\text{g}} = \textsc{Cost}(x^*_d(\cdot), x^*_s$)
      \State $J_{\text{c}} = \textsc{Cost}(x^*_s, n_s.state) +
n_s.{cost\_to\_come}$
       \State $\mathcal{N}.\text{push}(\textsc{Node}(x^*_s, \mathcal{F}_s,n_s,J_{\text{c}},J_{\text{g}}))$
       \If {$J_{\text{g}} + J_{\text{c}} < c}$ \Comment{Update cheapest node}
          \State $c = J_{\text{g}} + J_{\text{c}}$
          \State $n_c = n_s$
       \EndIf
       \If {$x^*_d(\cdot) = \mathcal{F}_d(x^*_s, \xi_d, \cdot)$} \Comment{Check for goal}
          \State $J_{\text{c}} = \textsc{Cost}(x^*_d, n_s.state) + n_s.{cost\_to\_come}$
          \State $n_d = \textsc{Node}(x^*_d, \mathcal{F}_d,n_s,J_{\text{c}},0))$ 
          \State $\mathcal{R} = \{n_d\}$
          \State \textbf{break;}
       \EndIf
    \EndIf
  \EndWhile
  \State $n = n_d$
  \While{$n.\text{parent} \neq \emptyset$} \Comment{Build final path}
     \State $\mathcal{R}.\text{push\_front}(n)$
     \State $n = n.\text{parent}$
  \EndWhile
  \State \Return $\mathcal{R}$
\EndFunction
\end{algorithmic}
\end{small}
\label{alg:rrt_search}
\end{algorithm}

The first planning step, an RRT-inspired search, randomly expands a tree
structure of nodes from an initial state ($x_0$), exploring the graph space
until a feasible path to the desired primitive ($\motionprimitive_p(\xi)$) is
reached.

At each iteration, a node from the set of explored nodes is sampled. With
probability $p$, this node is selected to be the cheapest node to encourage
exploration towards the goal, otherwise the sampled node is taken uniformly.
Next, a random motion primitive transfer function, corresponding arguments, and times
are selected uniformly from their respective domains. If the sampled
motion primitive transfer function makes progress, a new node is added 
as a child of the sampled node. As each new node is added, the cheapest node is
updated, and one-step reachability of the goal is checked
via $\mathcal{F}_d$, the transfer function associated with the desired
motion primitive, $\motionprimitive_d(\xi_d)$.

Once the goal can be reached, the main iteration loop terminates, and the
feasible path is returned. This is elucidated in Algorithm
\ref{alg:rrt_search}.
\subsection{Feasible Path Post-Processing}
The feasible trajectory is then post-processed via two components: a constrained
gradient descent and node-pruning. The constrained gradient descent
intends to select the locally optimal $\xi, t_0$, and $\Delta t$ to minimize
cost between nodes. Consider node $n_i$ with parent $n_{i-1}$ and child
$n_{i+1}$. For differentiable cost function $J:\mathcal{X} \times \mathcal{X} \rightarrow \mathbb{R}$, we have
local cost of $J_i \in \mathbb{R}$ as:
\begin{align}
J_i = J(x^*_{i-1},x^*_i) + J(x^*_{i},x^*_{i+1})
\end{align}
and cost gradient as:
\begin{align*}
\nabla J_i &=\frac{\partial J_i}{\partial(\xi_i,t_i,\Delta t_i)}
           =\frac{\partial J_i}{\partial x^*_i}\frac{\partial
x^*_i}{\partial(\xi_i,t_i,\Delta t_i)} \\
           &=\left(\frac{\partial J(x^*_{i-1},x^*_i)}{\partial x^*_i}
          +\frac{\partial J(x^*_i,x^*_{i+1})}{\partial x^*_i}\right)\frac{\partial
x^*_i}{\partial (\xi_i,t_i,\Delta t_i)} 
\end{align*}

As $J$ is differentiable and $x^*$ is differentiable in $\xi,t,\Delta t$, this gradient
exists and is well-defined. Constraining $x^*_i \in \mathcal{S}_{i+i}$, we can
apply constrained gradient descent \cite{boyd2003subgradient} to find the locally optimal choice
for $\xi_i,t_i$, and $\Delta t_i$ for each node.

In node-pruning, unnecessary nodes in the feasible path are removed. That is, for candidate unnecessary node $i$ in path $\mathcal{R}$, if:
\begin{align*}
x^*_1(x_0,\xi_1,t_1) &= \mathcal{F}_1(x_0, \xi_1,t_1,\Delta t_1)\\[-5pt]
&\;\;\vdots \\[-5pt]
x^*_{i+1}(x^*_{i-1},\xi_{i+1},t_{i+1}) &= \mathcal{F}_{i+1}(x^*_{i-1}, \xi_{i+1},t_{i+1},\Delta t_{i+1})\\[-5pt]
&\;\;\vdots \\[-5pt]
x^*_n(x^*_{n-1},\xi_n,t_n) &= \mathcal{F}_n(x^*_{n-1}, \xi_n,t_n,\Delta t_n)
\end{align*}
is still a feasible path, then node $n_i$ can be bypassed and should be removed
from $\mathcal{R}$. Each node in $\mathcal{R}$ is checked sequentially for
this condition and removed as necessary.

Note that the constrained gradient descent and path pruning post-processing
steps are coupled, and thus are intermingled iteratively to complete the
post-processing of the feasible path. The entirety of this process is depicted
in Figure~\ref{fig:graph_search}. In practice, many paths can be computed in parallel, and the lowest cost among the paths taken as the result.
\begin{figure}
  \centering
  \begin{subfigure}[t]{0.7\columnwidth}
      \centering
      \includegraphics[width=0.68\columnwidth]{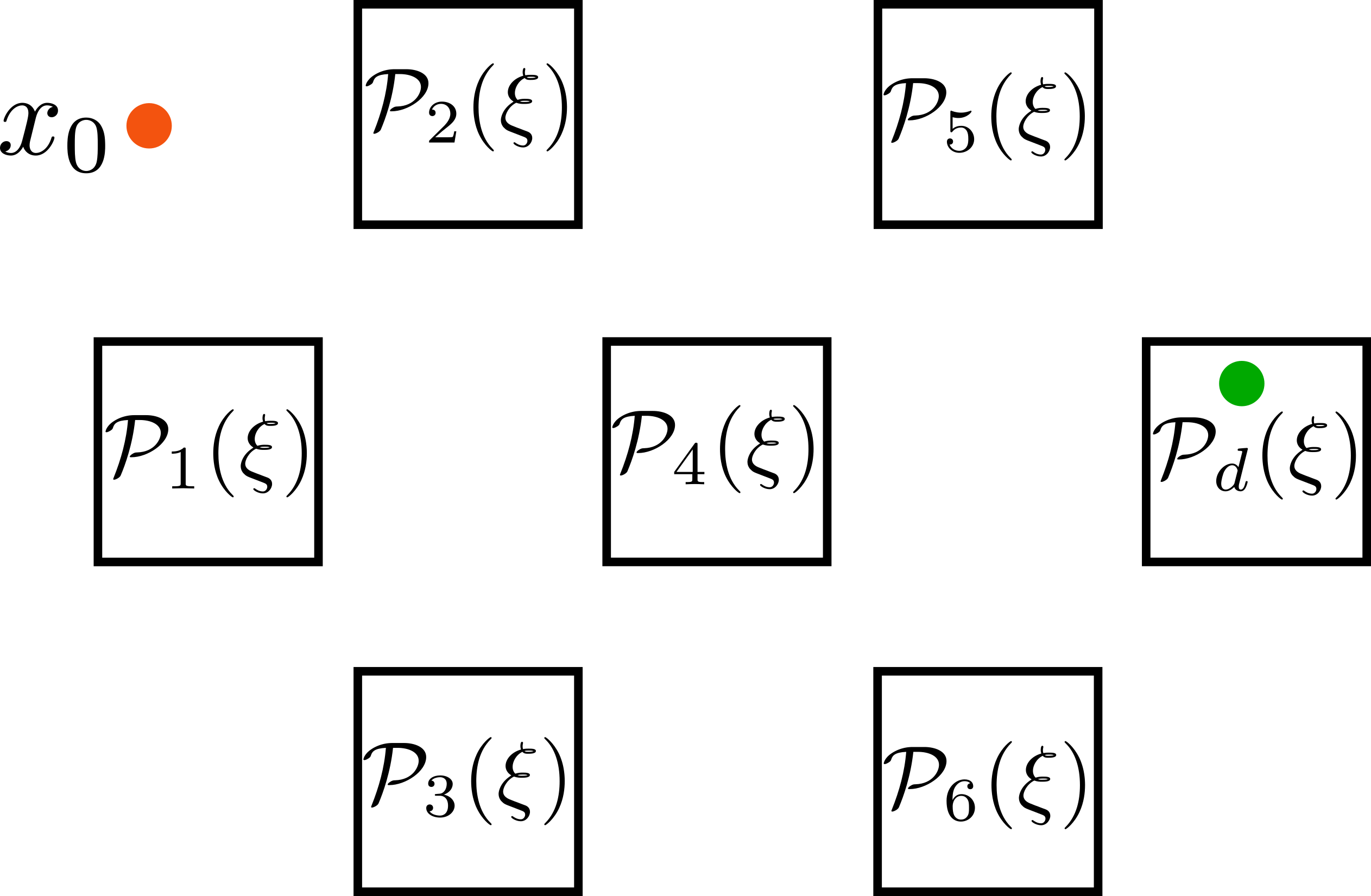}
      \captionsetup{format=hang}
      \caption{Motion Primitive Graph with initial state $x_0$ and goal primitive
$\mathcal{P}_d(\xi)$}
  \end{subfigure}\\
  \vspace{6mm}
  \begin{subfigure}[t]{0.475\columnwidth}
      \raggedright
      \includegraphics[width=\columnwidth]{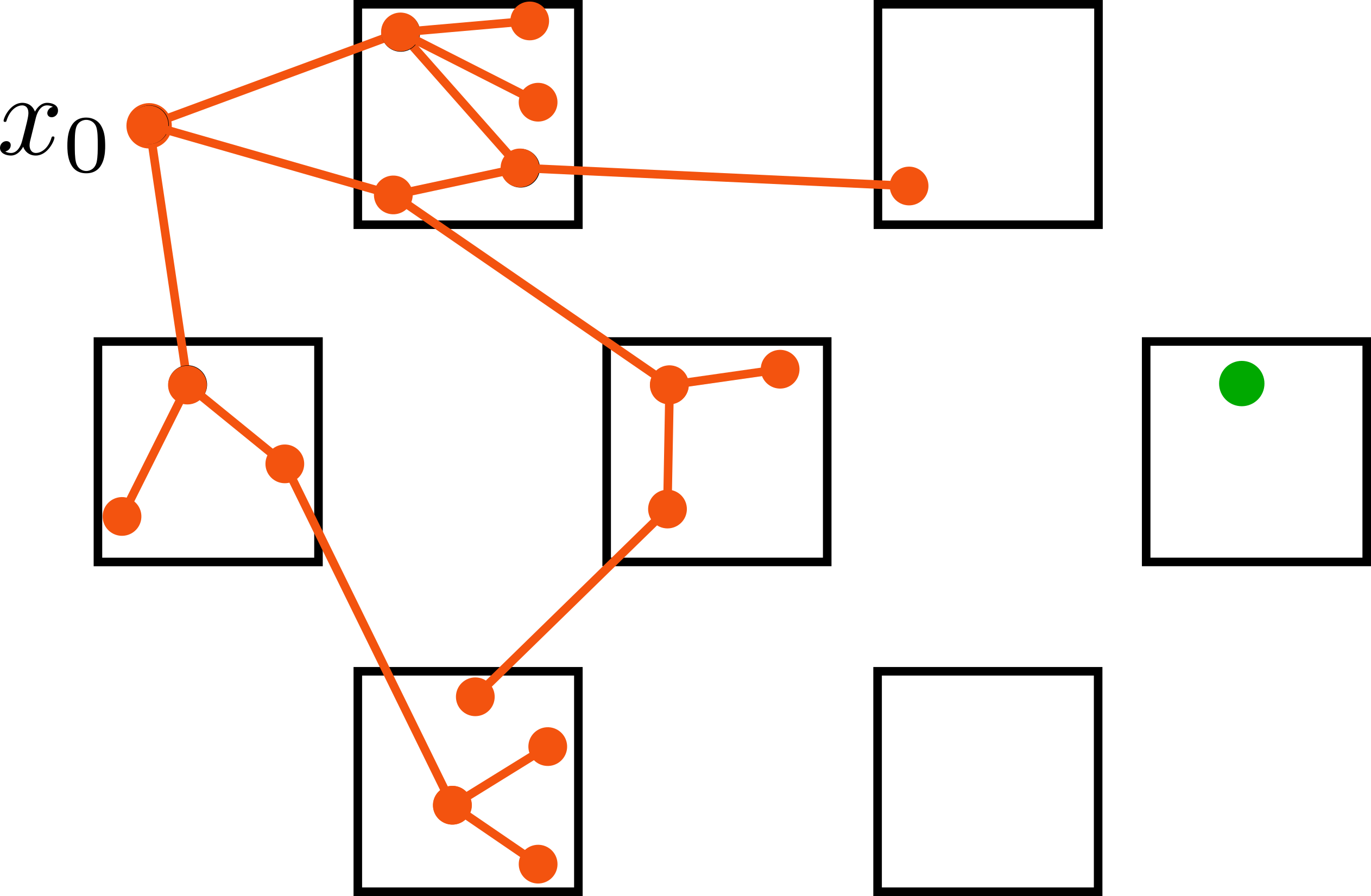}
      \centering
      \caption{Feasible Path Search}
  \end{subfigure}%
  \begin{subfigure}[t]{0.475\columnwidth}
      \raggedleft
      \includegraphics[width=\columnwidth]{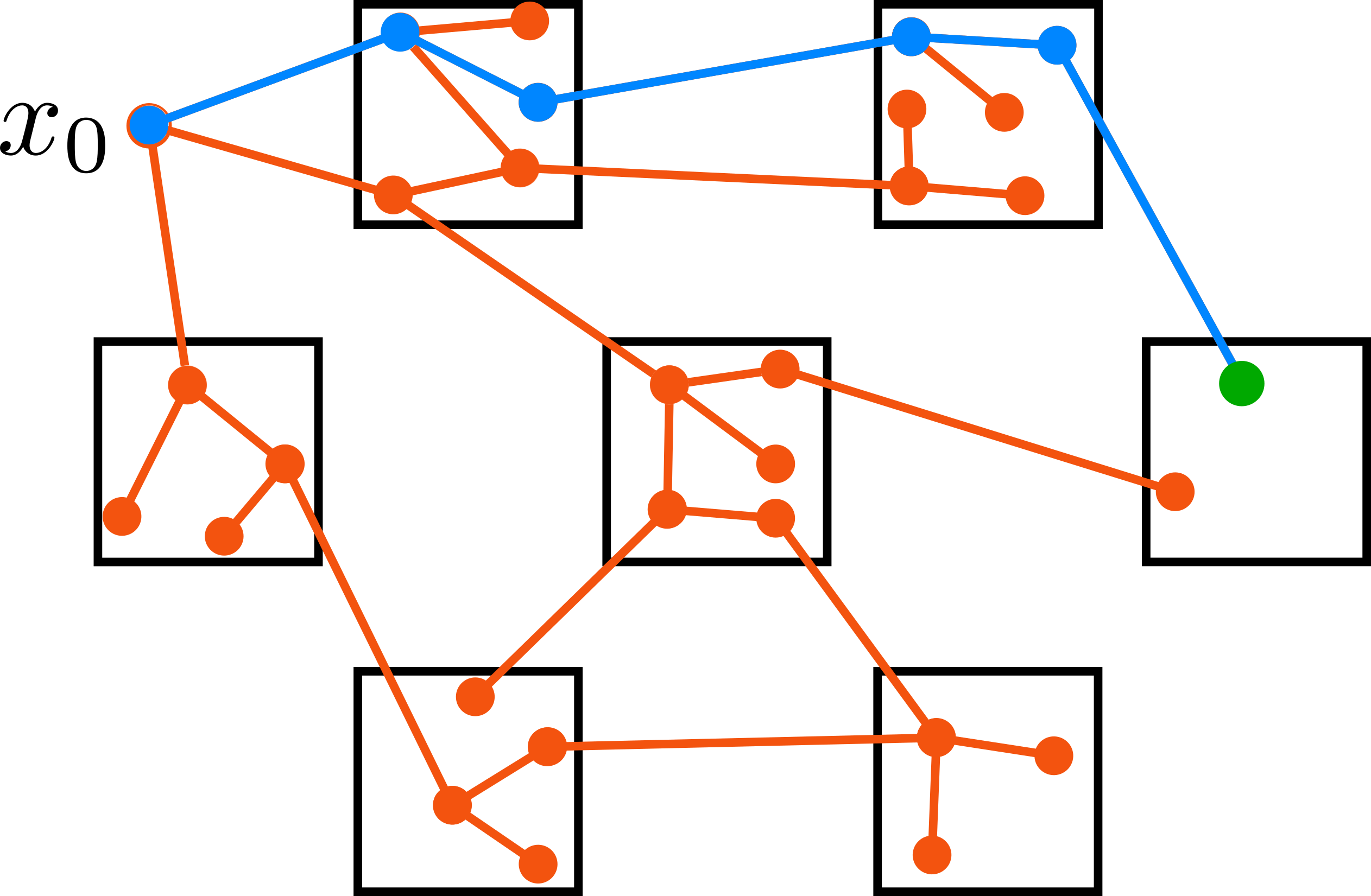}
      \centering
      \caption{Feasible Path Found}
  \end{subfigure}\\
  \vspace{6mm}
  \begin{subfigure}[t]{0.475\columnwidth}
      \raggedright
      \includegraphics[width=\columnwidth]{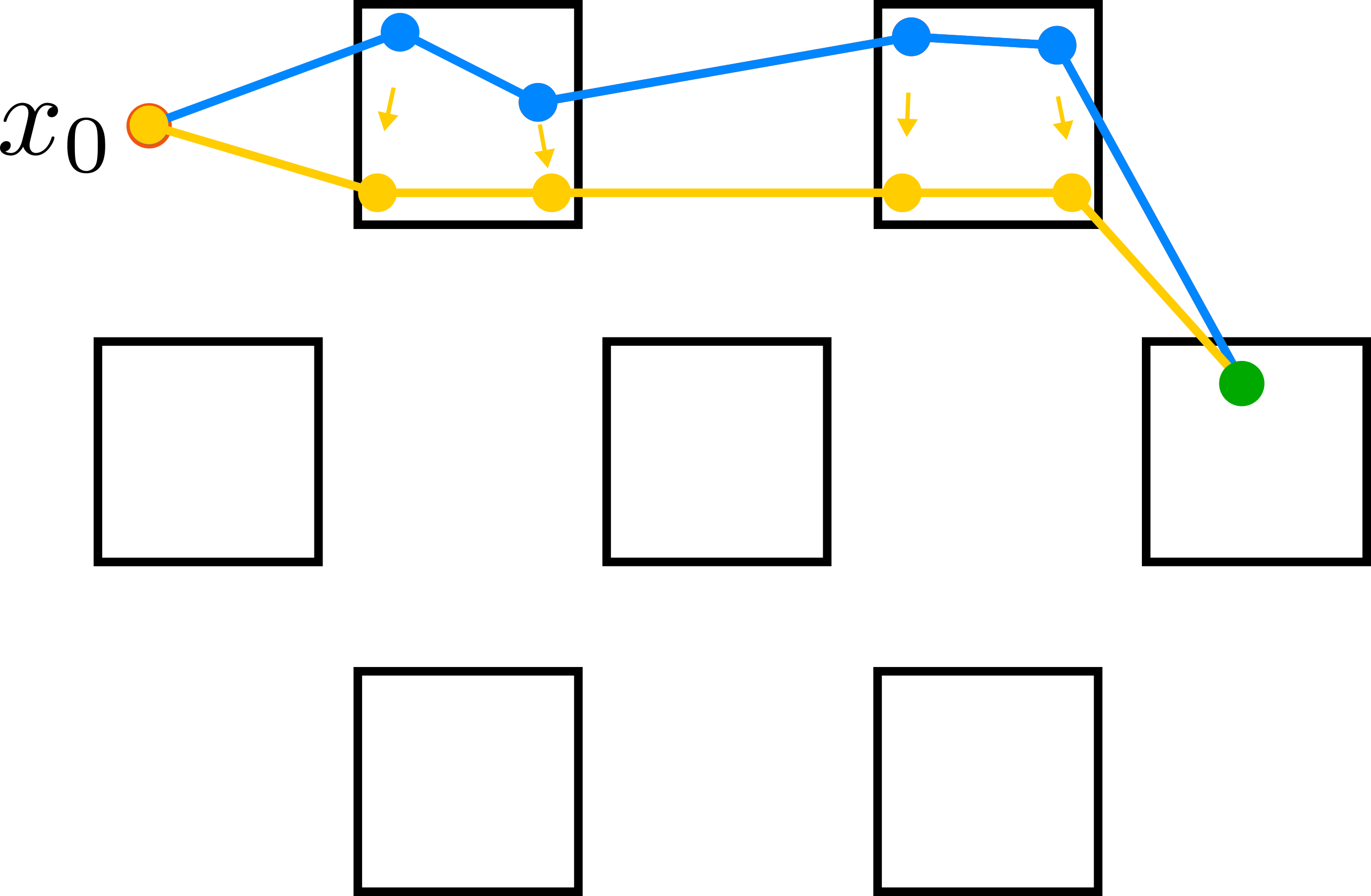}
      \captionsetup{format=hang,width=0.8\columnwidth}
      \caption{Constrained Gradient Descent in $\xi$}
  \end{subfigure}%
  \begin{subfigure}[t]{0.475\columnwidth}
      \raggedleft
      \includegraphics[width=\columnwidth]{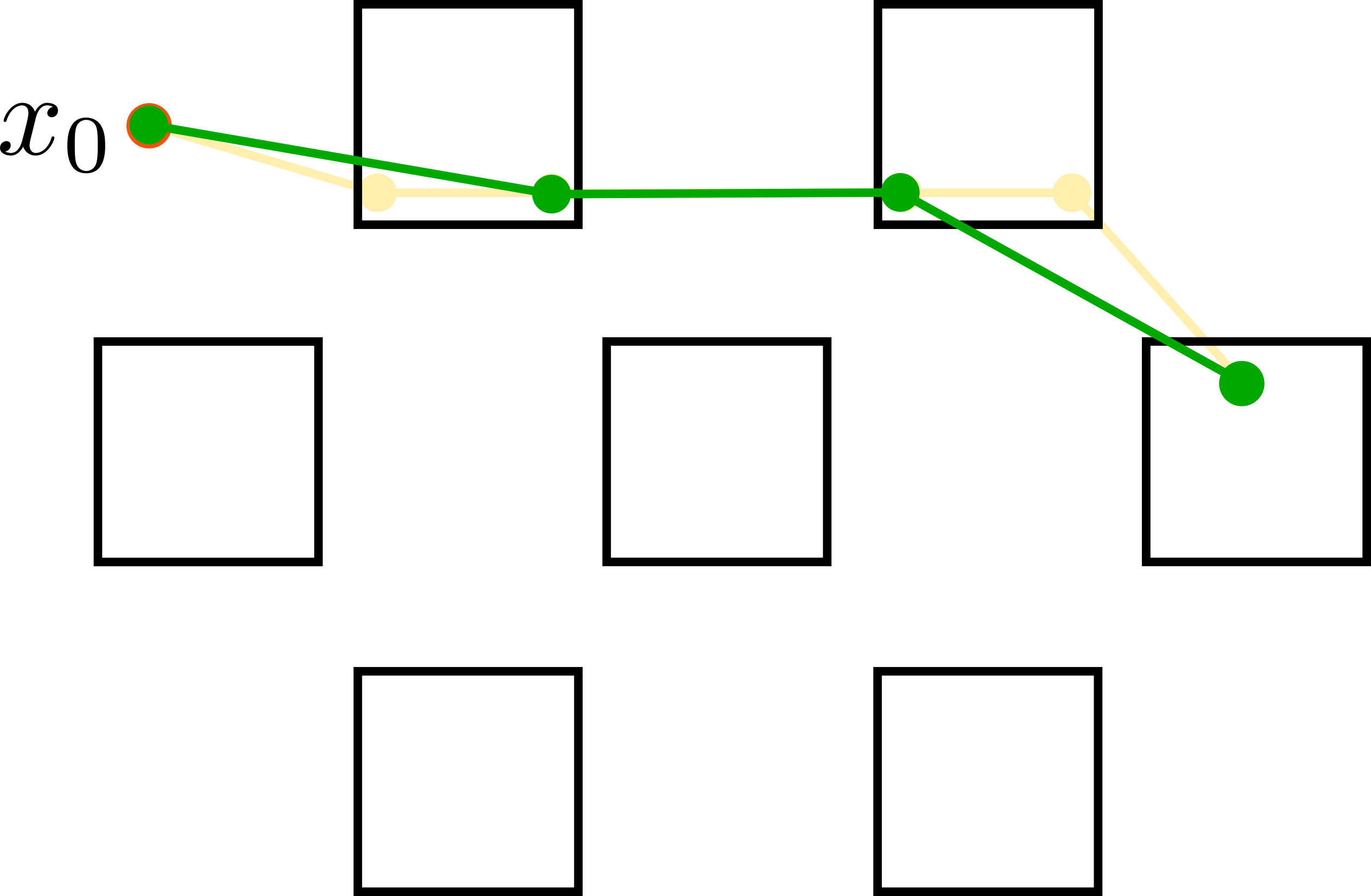}
      \captionsetup{format=hang,width=0.95\columnwidth}
      \caption{Prune unnecessary nodes, \\ Final path found}
  \end{subfigure}
  \caption{Depiction of the search algorithm for the mixed discrete and
continuous motion primitive transition graph. Boxes represent
motion primitives and their continuous domain of arguments. 
Step (d) and (e) are iterated together.}
  \label{fig:graph_search}
\end{figure}

\section{Application to Quadrupeds}
\label{sec:quadruped}

To investigate this work in a real-world application, the presented concepts are
applied to the Unitree A1 quadrupedal robot with a experimental set of motion
primitives. Here, we have configuration space $q \in \mathcal{Q} \subset
\mathbb{R}^n$ with state space $x = (q,\dot{q}) \in \mathcal{X} =T\mathcal{Q}\subset
\mathbb{R}^{2n}$ with $n=18$. We have $m=12$ actuated degrees of freedom for
control input $u\in \mathcal{U} \subset \mathbb{R}^m$. The
\textit{hybrid-dynamic} nature
of the system leads to several domains of operation to be considered. These
domains are marked by the contact state of each foot, denoted by a contact
vector $c \in \{0,1\}^{|\mathcal{N}_c|}$ where $\mathcal{N}_c = \{1,2,3,4\}$ the set
of considered contacts, in
this case the quadruped's feet. We consider and model no-slip via \textit{holonomic
constraints} $\psi(q) \equiv 0, \psi(q) \in \mathbb{R}^{h}$, where $h$ depends on the
number of active contacts, i.e. the hybrid domain. 
We have our system dynamics for a specific domain in control affine form as:
\begin{align*}
    \dot{x} = \underbrace{\begin{bmatrix}\dot{q}\\-D(q)^{-1}(H(q,\dot q)-J(q)^\top \lambda)\end{bmatrix}}_{f(x)} + \underbrace{\begin{bmatrix}0\\D(q)^{-1}B\end{bmatrix}}_{g(x)}u,
\end{align*}
%
%
where $D(q)\in\mathbb{R}^{n\times n}$ is the mass-inertia matrix, $H(q,\dot
q)\in\mathbb{R}^n$ accounts for the Coriolis and gravity terms,
$B\in\mathbb{R}^{n\times m}$ is the actuation matrix, $J(q) = \frac{\partial
c(q)}{\partial q} \in \mathbb{R}^{h}$ is the Jacobian of the holonomic
constraints, and $\lambda\in \mathbb{R}^{h}$ is the constraint wrench.
${f:\mathcal{X}\to \mathcal{X}}$ and ${g:\mathcal{X}\to\mathbb{R}^{{2n}\times m}}$ are assumed to be locally Lipschitz continuous. 

\subsection{Quadruped Motion Primitives Preliminaries}
Our experiments include several experimental motion primitives that utilize
various control techniques to achieve their desired behavior. Though our
method is agnostic to these implementation details and only requires that
Definition~\ref{def:primitive} be satisfied, a brief discussion
provides valuable context for realizing our method on a real system. We will
begin by addressing some commonality between our test motion primitives.

\begin{primitive}{Position and Velocity Safe Sets}
For all primitives, we can define the safe set for joint position and
velocity limits as:
\begin{align*}
\mathcal{C}_{q,\dot{q}} = \{x \in \mathcal{X} : \;&  q_{\rm min} \leq q \leq q_{\rm max},
                                                  &  \dot{q}_{\rm min} \leq \dot{q} \leq \dot{q}_{\rm max} \}.
\end{align*}
\end{primitive}
\begin{primitive}{Computing Safe Regions of Attraction}
We consider an estimate $\mathcal{E}
\subseteq \mathcal{S}$. There are several methods available to build
$\mathcal{E}$, include Lyapunov-based methods \cite{davison1971computational,johansen2000computation, polanski1997lyapunov}, and backwards reachability
analysis \cite{Yuan2019}. 
We employ a conservative construction from Lyapunov analysis
of the linearization of the controller \cite{khalil2002nonlinear} and expand this region via
hardware testing until the states of nominal operation are encompassed to produce
conservative, but useful, formulations for $\mathcal{E}$.
\end{primitive}

\subsection{Quadruped Motion Primitives in Experiments}
\label{subsec:quadruped_primitives}

\begin{primitive}{Lie}
\textit{Lie} is a motion primitive that rests the quadruped on the ground with the
legs in a prescribed position. The feedback controller is a joint-space PD controller
where $x^*(x_0,t)$ is a cubic spline motion profile from the
initial pose to the goal pose, $x^*_{\rm Lie}$. There are no
continuous arguments, $\Xi = \emptyset$. In addition to the common safe set, 
the safe set for Lie requires at least one foot in
contact with the ground, i.e. 
\begin{align*}
\mathcal{C}_{\rm Lie} = \mathcal{C}_{q,\dot{q}} \cap \{x \in \mathcal{X} \;|\; \exists
n_c \in \mathcal{N}_c \text{ where } c\{n_c\} = 1\}
\end{align*}
\end{primitive}

\begin{primitive}{Stand}
The \textit{Stand} motion primitive has setpoint $x^*_{\rm Stand}(\xi,t)$ 
to drive the body to specified height and orientation and 
center of mass to be above the centroid of the support polygon. Its trajectory is determined by a cubic spline in center of mass task-space. $\Xi =
\{h,\theta_x,\theta_y,\theta_z\}$ with domain between bounds
$\xi_{\rm min}, \xi_{\rm max}$ derived from kinematic limits.

The control law is an Inverse-Dynamics Quadratic Program (ID-QP) including no-slip constraints on the feet.
$\ddot{q}(x,t)$ in the objective function is
specified by a task-space PD control law.
Since this controller assumes ground contact of all feet, we require it via the safe set:
\begin{align*}
\mathcal{C}_{\rm Stand} = \mathcal{C}_{q,\dot{q}} \cap \{x \in \mathcal{X} \;|\;
c\{n_c\} = 1 \quad \forall n_c \in \mathcal{N}_c \}
\end{align*}
\end{primitive}

\begin{primitive}{Walk}
The \textit{Walk} primitive is a diagonal-gait walking trot, with arguments $\Xi
= \{h,v_x,v_y,v_{\theta_z}\}$ and associated bounds corresponding to 
linear velocity in $x$ and $y$, angular velocity
about the $z$ axis, and body height.
As with \textit{Stand}, walk uses an ID-QP based controller to track $x_{\rm walk}^*(\xi,t)$ in center of
mass space, but in this case there is an additional component for a Raibert-style swing leg
trajectory \cite{raibert1986legged}.

This controller assumes contact of the diagonal stance legs, so we have safe set as:
\begin{align*}
\mathcal{C}_{\rm Walk}(t)= \mathcal{C}_{q,\dot{q}} \cap \{x \in \mathcal{X}
\;|\; &c\{n_{c}\} = 1 \quad \forall n_c \in \mathcal{N}_{\rm stance}(t) \}
\end{align*}
where $\mathcal{N}_{\rm stance}(t) \subset \mathcal{N}_c$ are the stance contacts
at time $t$.

\end{primitive}

\begin{primitive}{Land}
The \textit{Land} primitive is a high-damping task-space PD control law on the
position of the feet while the quadruped is airborne. The goal position of the
feet is specified to maintain a constant support polygon, but the relative
position of the center of the support polygon with respect to the center of mass
is modulated according to a spring-loaded inverted pendulum model (SLIP) to
remove all body velocity during the contact phase.
The setpoint $x^*(x_0,t)$ is derived from a ballistic trajectory
given by the initial position and velocity and kinematics for the desired foot
pose. The safe set for Land is simply the joint position and velocity safe set,
$\mathcal{C}_{\rm Land} = \mathcal{C}_{q,\dot{q}}$.
\end{primitive}
\subsection{Implementation and Experimental Results}
Implementations for each experimental motion primitive and the algorithms
described in Section~\ref{sec:contribution} were built in our C++ motion
primitive control framework. Here, the main control loop process runs 1kHz,
and the motion primitive graph search runs in a separate thread asynchronously,
with typical computation time less than 50 ms.
The computation is done on a onboard Intel NUC with an i7-10710U CPU and 16GB of
RAM.
\begin{figure*}[t]
  \centering
  \begin{subfigure}[t]{\textwidth}
    \includegraphics[width=0.75\columnwidth] {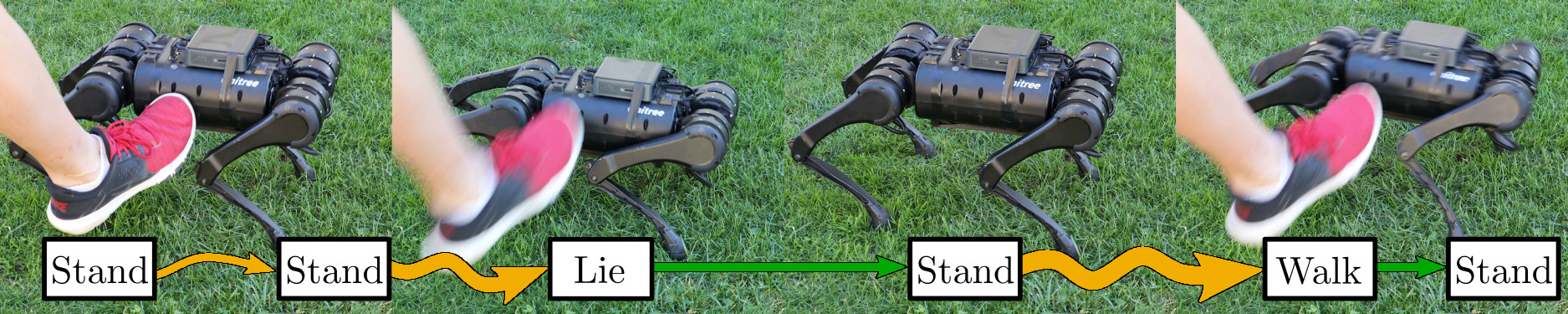}%
    \hspace{1mm}
    \includegraphics[width=0.22\columnwidth,trim=0 7mm 0 7mm]{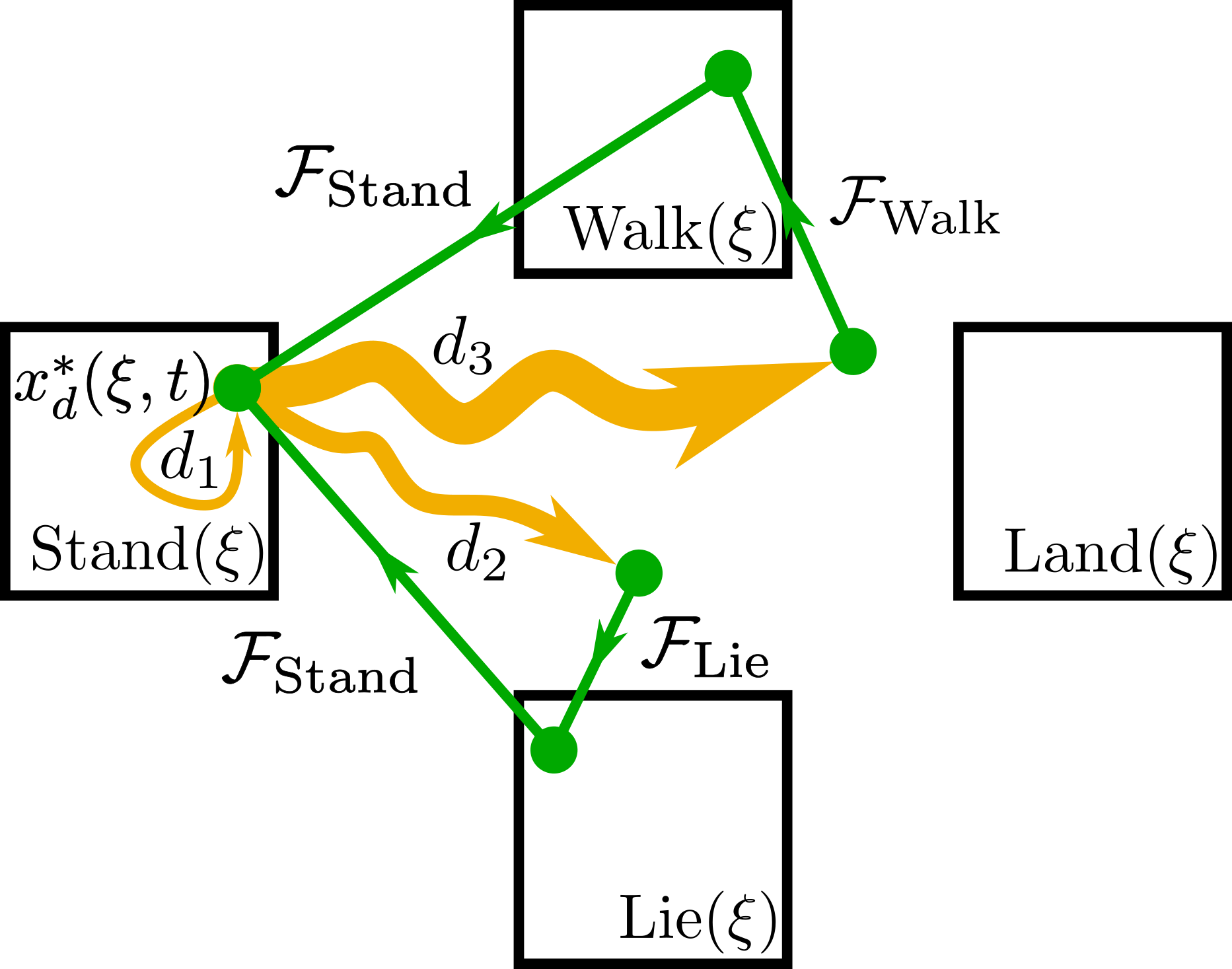}
    \caption{\textbf{Kick During Stand}: Differing magnitudes of disturbance elicit different responses.}
    \label{fig:stand_kick}
  \end{subfigure}
  \begin{subfigure}[t]{\textwidth}
    \includegraphics[width=0.75\columnwidth] {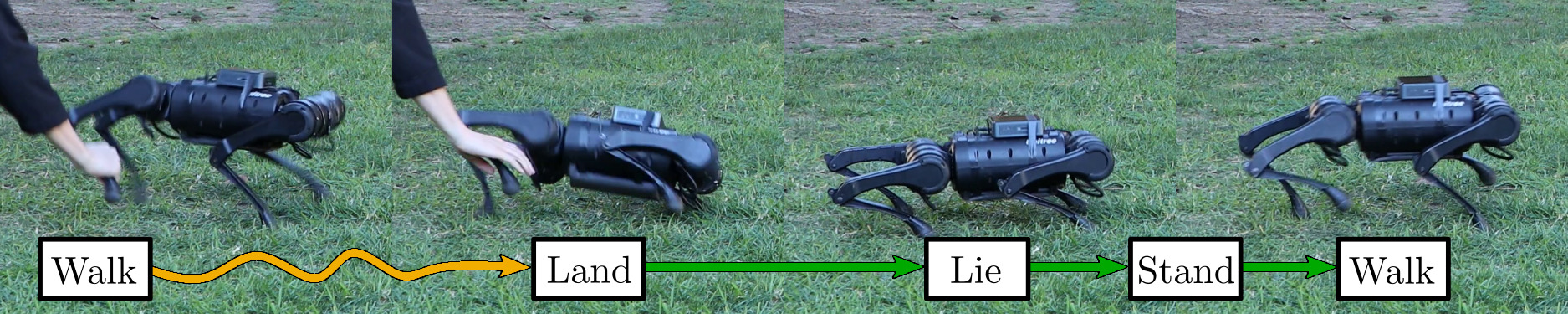}%
    \hspace{1mm}
    \includegraphics[width=0.22\columnwidth,trim=0 7mm 0 7mm]{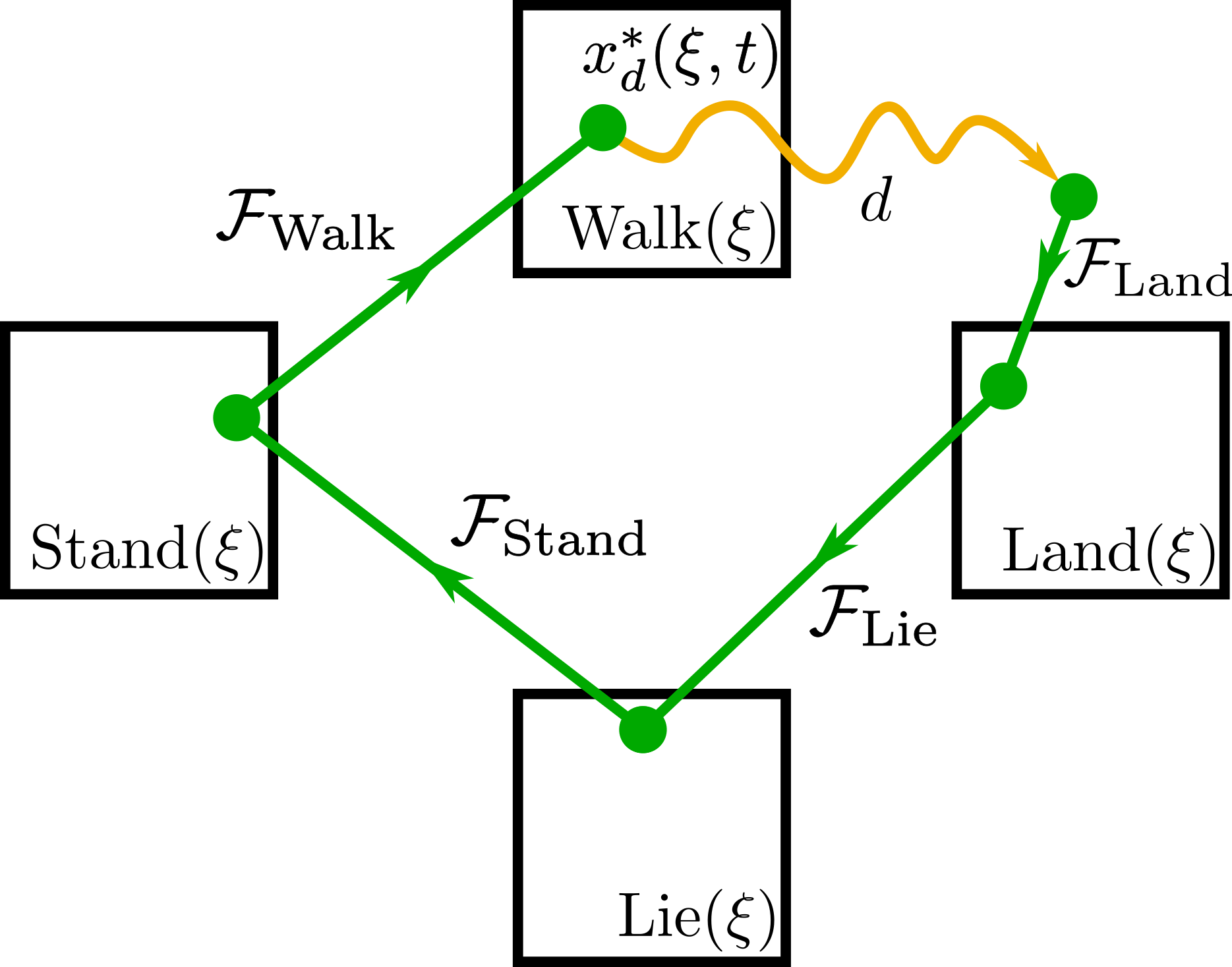}
    \caption{\textbf{Leg Pull while Walking}: Robustness to intentional disturbance while walking.}
    \label{fig:leg_grab}
  \end{subfigure}
  \begin{subfigure}[t]{\textwidth}
    \includegraphics[width=0.75\columnwidth] {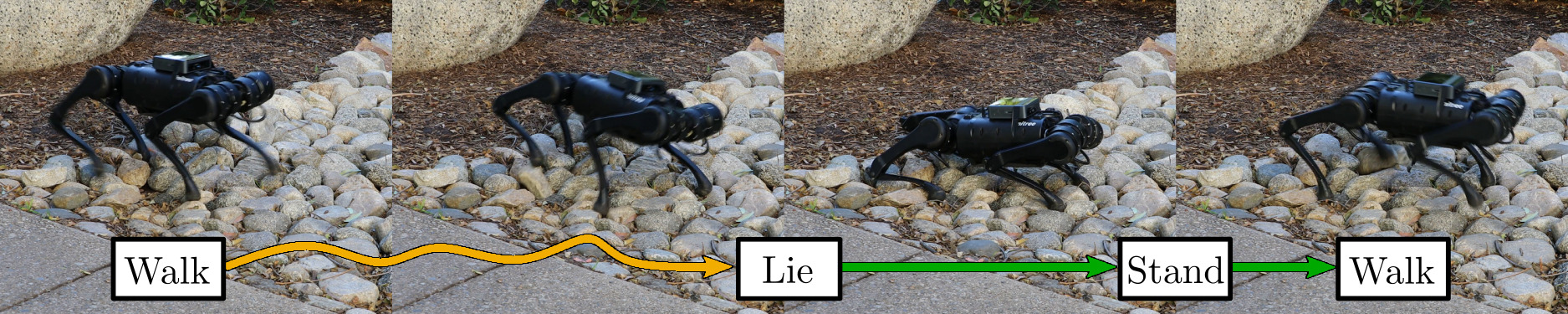}%
    \hspace{1mm}
    \includegraphics[width=0.22\columnwidth,trim=0 7mm 0 7mm]{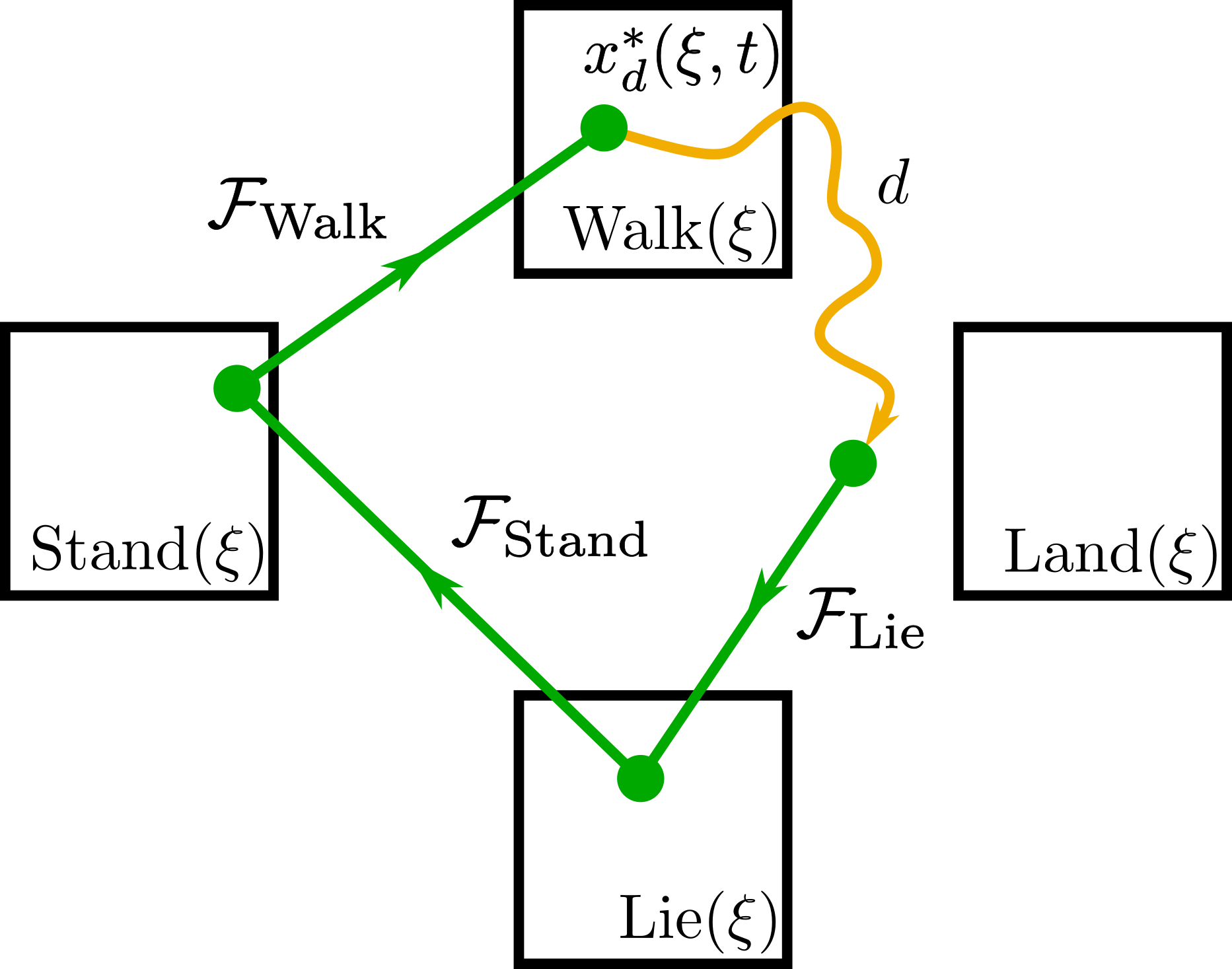}
    \caption{\textbf{Walking on Loose, Uneven Stones}: Robustness to challenging walking environment}
    \label{fig:stones}
  \end{subfigure}
  \begin{subfigure}[t]{\textwidth}
    \includegraphics[width=0.75\columnwidth] {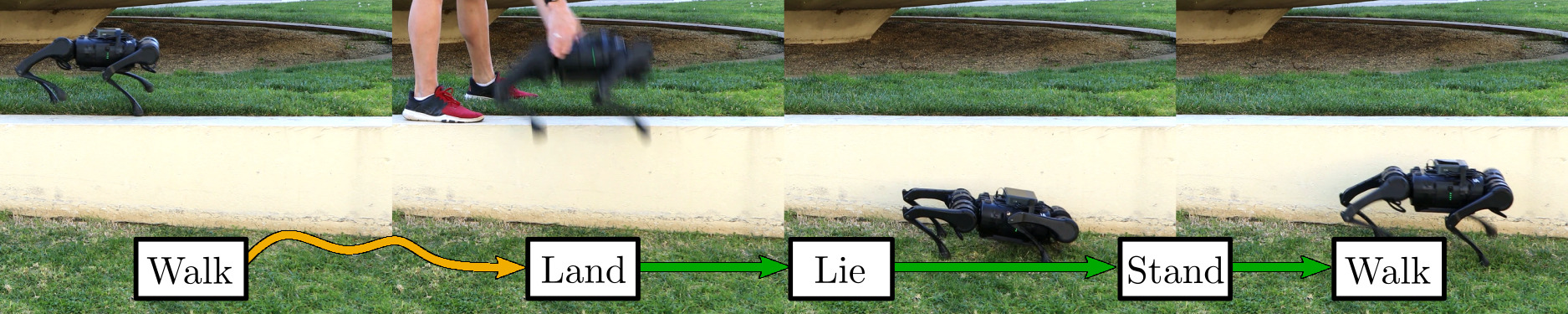}%
    \hspace{1mm}
    \includegraphics[width=0.22\columnwidth,trim=0 7mm 0 7mm]{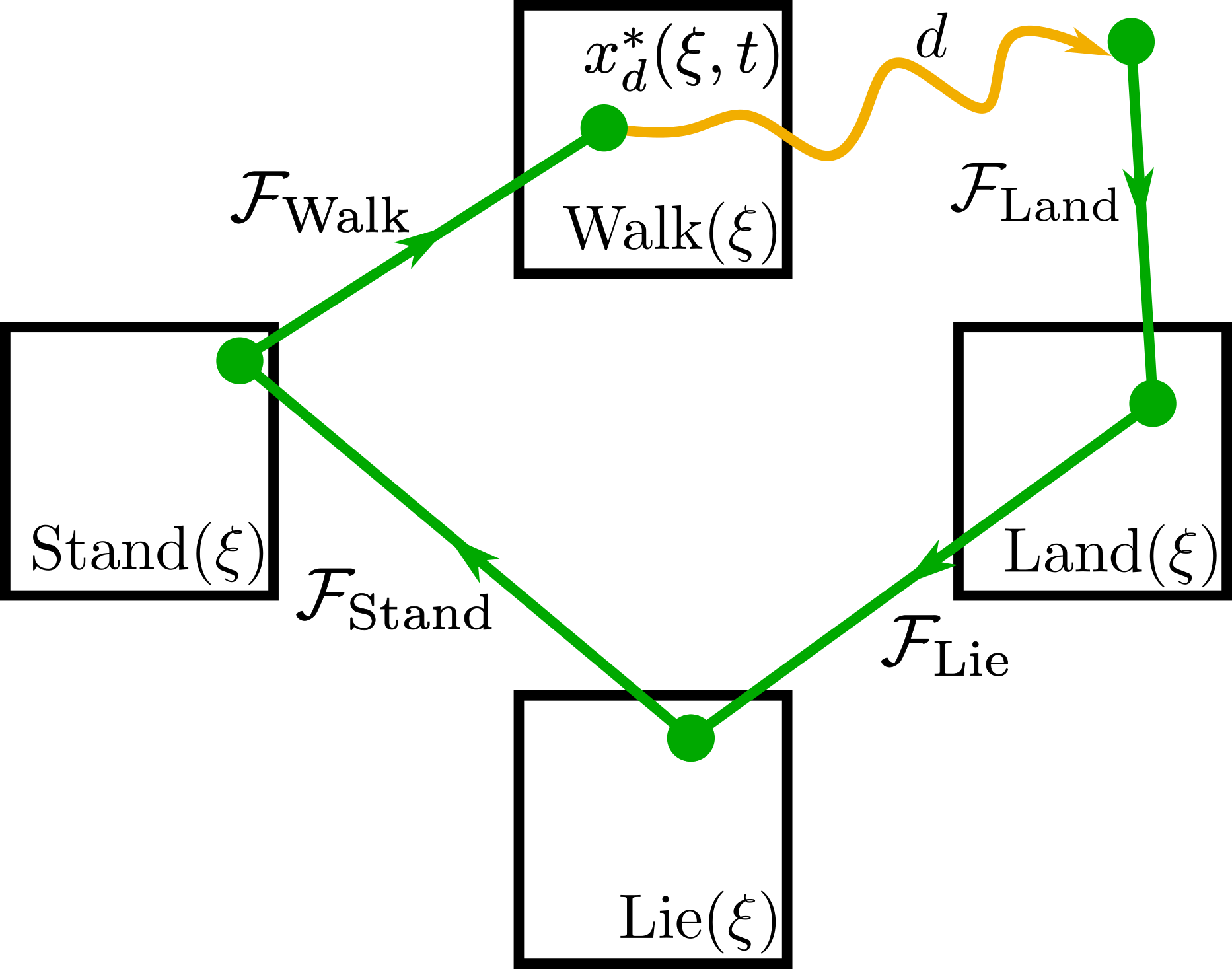}
    \caption{\textbf{Ledge Toss}: Combination of disturbance and large environmental uncertainty}
    \label{fig:ledge_fall}
  \end{subfigure}
  \caption{The experimental results of our proposed method exhibiting robustness across a variety of disturbances and conditions. Video of these results can be seen in the supplemental video \cite{video}.}
  \label{fig:exp_results}
  \vspace{-2mm}
\end{figure*}

The specifics and results of each experiment are discussed below with details 
in Figure~\ref{fig:exp_results} and supplementary video
\cite{video}. Motion primitive commands are truncated for brevity when arguments
are equal to zero, i.e. Stand($h=0.2$ m, $\theta_x = 0$
rad,  $\theta_y = 0$ rad, $\theta_z = 0$ rad) is shown as Stand($h=0.2$ m).

\begin{primitive}{Nominal Transition to Walk}
Though emphasis in this work is achieving robustness via motion primitive
transitions, it is implied that a nominal transition should be successful. In
this experiment, the initial pose of the robot is at rest on the ground, with
the motors unactuated. The commanded motion primitive is Walk$(h = 0.25\textrm{
m}, v_x = 0.2 \textrm{ m/s})$. From this position, the algorithm computes a sequence
from the initial state to Lie(), Stand($h=0.2$ m), and then the goal, Walk$(h =
0.25\textrm{ m}, v_x = 0.2\textrm{ m/s})$.
\end{primitive}
\begin{primitive}{Kick During Stand}
The command motion primitive is Stand($h = 0.25$ m), and subject to kick
disturbances of varying magnitude. The Stand primitive has some inherent
robustness, and when a small kick is applied, the state remains within the safe
region of attraction and the system is stable to the setpoint without any
transition. With a moderate kick, this is not the case, and the algorithm
transitions through Lie before returning to standing at the desired height. With
an even larger kick, a different plan is computed, transitioning to walking in
place with Walk($h=0.25$) before returning to Stand($h = 0.25$ m).
\end{primitive}
\begin{primitive}{Leg Pull while Walking}
In this test, the desired primitive is Walk($h=0.25$ m, $v_x = 0.2$ m/s).
During the walk, the operator grabs a rear leg of the quadruped, providing some
initial disturbance and preventing forward motion during the leg's swing phase.
In response, the sequence Land(), Stand($h=0.2$ m), Walk($h=0.25$ m, $v_x = 0.2$
m/s) is computed. However, the continued disturbance prevents Lie() from being
executed, and the system stays in the Land() primitive until the leg is
released. At this point the recomputed plan can be finished and the quadruped
resumes walking.
\end{primitive}
\begin{primitive}{Walking on Loose, Uneven Stones}
The robot is commanded Walk($h=0.25$ m, $v_x = 0.2$ m/s), but
with a challenging environment consisting of loose, uneven stones. The stones
cause the footfall height to vary across steps, and can be move when stepped on 
causing further deviation from the expected conditions. There is no
perception involved and the walking primitive assumes flat terrain. In this
test, the quadruped is able to progress slowly, taking steps and planning
through the disturbances as they are
encountered. Replans include transitioning through Lie() and Stand($\xi$)
back to Walk, a single Stand($\xi$) to Walk, and in some cases, Land(),
Lie(), Stand($\xi$), Walk($h=0.25$ m, $v_x = 0.2$ m/s).
\end{primitive}
\begin{primitive}{Ledge Toss}
In this experiment, the quadruped is tossed off an $\sim0.5$ m high ledge
while being commanded to Walk($h=0.25$ m, $v_x = 0.2$ m/s). As the feet leave
contact with the ground, the Land() primitive begins executing, and continues to
execute until the state allows the transition to continue through Lie, Stand,
and back to the desired Walk command.
\end{primitive}

\section{CONCLUSIONS}
\label{sec:concl}

Motivated by the desire to achieve robust autonomy on dynamic robots, this
paper has established a definition of \textit{motion primitives} and used these
attributes to construct an abstraction of the dynamics through the
\textit{motion primitive transfer function}. This formulation leads to a
mixed discrete and continuous graph structure and we
presented a probabilistic search algorithm with constrained gradient descent and
node pruning post-processing to search this space for transition paths. The
performance of this procedure allows for online replanning of paths through the
motion primitive graph and can be used to reach the goal primitive in both
nominal and disturbed scenarios. This was demonstrated on a quadrupedal robot
for several experimental motion primitives subject to a variety of environmental
and antagonistic disturbances.

While this represents a significant contribution to robust autonomy on dynamic
systems, there are a number of extensions the authors would like to pursue. 
While a probabilistic approach is widely used in high-dimensional search
problems and represents a natural starting point, there is additional structure
to the problem ignored in this approach. As such, we intend to investigate how
this structure may be incorporated into search to improve results. 
We would also like to extend this framework to the contexts with perception,
and consider obstacles and varying environments explicitly. Future work
intends address these avenues in the context of motion primitives and dynamic
autonomy.

\addtolength{\textheight}{-8cm}   






\bibliographystyle{IEEEtran}
\bibliography{reference}	

\end{document}